\documentclass{article}
\usepackage{linjmlr2e}
\usepackage{microtype}
\usepackage{graphicx}
\usepackage{subfigure}
\usepackage{booktabs} 
\usepackage{hyperref}

\usepackage{amssymb}
\usepackage{amsmath}
\usepackage{algorithm}
\usepackage{algorithmic}

\newcommand{\REAL}{\mathbb{R}}
\newcommand{\prob}[1]{\Pr\left[#1\right]}
\newcommand{\x}{\alpha}
\newcommand{\y}{\beta}
\newcommand{\z}{\gamma}

\usepackage{scalerel,stackengine}
\stackMath
\newcommand\reallywidehat[1]{%
\savestack{\tmpbox}{\stretchto{%
  \scaleto{%
    \scalerel*[\widthof{\ensuremath{#1}}]{\kern-.6pt\bigwedge\kern-.6pt}%
    {\rule[-\textheight/2]{1ex}{\textheight}}
  }{\textheight}%
}{0.5ex}}%
\stackon[1pt]{#1}{\tmpbox}%
}

\title{Categorical Feature Compression via Submodular Optimization}
\author{\name MohammadHossein~Bateni\thanks{Authors are listed in alphabetical order.} \email bateni@google.com \\ \addr Google Research 
\AND
\name Lin~Chen \email lin.chen@yale.edu\\ \addr Google Research and Yale University
\AND
\name Hossein~Esfandiari \email esfandiari@google.com\\
\addr Google Research
\AND
Thomas~Fu \email thomasfu@google.com\\
\addr Google Research 
\AND
\name Vahab~S.~Mirrokni \email mirrokni@google.com\\
\addr 
Google Research
\AND
\name Afshin~Rostamizadeh \email rostami@google.com\\
\addr Google Research}

\begin{document}

\maketitle

\begin{abstract}
In the era of big data, learning from categorical features with very large vocabularies (e.g., 28 million for the Criteo click prediction dataset) has become a practical challenge for machine learning researchers and practitioners.  We design a highly-scalable vocabulary compression algorithm that seeks to maximize the mutual information between the compressed categorical feature and the target binary labels and we furthermore show that its solution is guaranteed to be within a $1-1/e \approx 63\%$ factor of the global optimal solution.  To achieve this, we introduce a novel re-parametrization of the mutual information objective, which we prove is submodular, and design a data structure to query the submodular function in amortized $O(\log n )$ time (where $n$ is the input vocabulary size). Our complete algorithm is shown to operate in $O(n \log n )$ time. Additionally, we design a distributed implementation in which the query data structure is decomposed across $O(k)$ machines such that each machine only requires $O(\frac n k)$ space, while still preserving the approximation guarantee and using only logarithmic rounds of computation.  We also provide analysis of simple alternative heuristic compression methods to demonstrate they cannot achieve any approximation guarantee.  Using the large-scale Criteo learning task, we demonstrate better performance in retaining mutual information and also verify competitive learning performance compared to other baseline methods.
\end{abstract}

\section{Introduction}

In modern large scale machine learning tasks, the presence of
categorical features with extremely large vocabularies is a standard
occurrence. For example, in tasks such as product recommendation and
click-through rate prediction, categorical variables corresponding to
inventory id, webpage identifier, or other such high cardinality
values, can easily contain anywhere from hundreds of thousands to tens
of millions of unique values. The size of machine learning models
generally grows at least linearly with the vocabulary size and, thus,
the memory required to serve the model, the training and inference
cost, as well as the risk of overfitting become an issue with very
large vocabularies.  In the particular case of neural networks model,
one generally uses an embedding layer to consume categorical inputs.
The number of parameters in the embedding layer is $O(nh)$, where $n$
is the size of the vocabulary and $h$ is the number of units in the
first hidden layer.

To give a concrete example, the Criteo click prediction benchmark has
about 28 million categorical feature values \citep{criteo}, thus
resulting in an embedding layer more than 1 billion parameters for a
modestly sized first hidden layer. Note, this number dwarfs the number
of parameters found in the remainder of the neural
network. Again, to give a concrete example, even assuming a very deep
fully connected network of depth $10^2$ with hidden layers of size
$10^3$, we have $(10^3 \times 10^3) 10^2 = 10^8$ parameters in the hidden
network -- still an order of magnitude smaller than the embedding
layer alone.  This motivates the problem of compressing the vocabulary
into a smaller size while still retaining as much
information as possible.

In this work, we model the compression task by considering the problem
of maximizing the mutual information between the compressed version of
the categorical features and the target label.  We first observe a
connection between this problem and the quantization problem for discrete
memoryless channels, and note a polynomial-time algorithm for the
problem~\citep{kurkoski2014,iwata2014}. The resulting algorithm, however, is based on solving a
quadratic-time dynamic program, and is not scalable. Our main goal in
this paper is to develop a scalable and distributed algorithm with a
guaranteed approximation factor.  We achieve this
goal by developing a novel connection to submodular
optimization. Although in some settings, entropy-based set functions
are known to be submodular, this is not the case for the mutual
information objective we consider (mutual information with respect to
the target labels).  Our main insight is in proving the submodularity
of a particular transformation of the mutual information objective,
which still allows us to provide an approximation guarantee on the
quality of the solution with respect to the original objective. We
also provide a data structure that allows us to query this newly
defined submodular function in amortized logarithmic time. This
logarithmic-time implementation of the submodular oracle empowers us
to incorporate the fastest known algorithm for submodular
maximization~\citep{mirzasoleiman2015}, which leads us to a sequential
quasi-linear-time $(1-1/e-\epsilon)$-approximation algorithm for
binary vocabulary compression. Next, we provide a distributed
implementation for binary vocabulary compression. Previous
distributed algorithms for submodular maximization assume a direct
access the query oracle on every machine (e.g.,
see~\citep{barbosa2015power,mirrokni2015,mirzasoleiman2013}). However, the query oracle
itself requires $O(n)$ space, which may be restrictive in the large
scale setting. In this work, we provide a truly distributed implementation
of the submodular maximization algorithm
of~\citep{badanidiyuru2014fast} (or similarly ~\citep{kumar2015fast}) for our application by distributing the query oracle. In this
distributed implementation we manage to decompose the query oracle
across $O(k)$ machines such that each machine only requires $O(\frac n k)$
space to store the partial query oracle. As a result, we successfully
provide a distributed $(1-1/e-\epsilon)$-approximation algorithm for
vocabulary compression in logarithmic rounds of computation.  Our
structural results for submodularity of this new set function is the
main technical contribution of this paper, and can also be of
independent interest in other settings that seek to maximize mutual
information.

We also study a number of heuristic and baseline algorithms for the
problem of maximizing mutual information, and show that they do not
achieve a guaranteed approximation for the problem.
Furthermore, we study the empirical performance
of our algorithms on two fronts: First, we show the effectiveness of
our greedy scalable approximation algorithm for maximizing mutual
information. Our study confirms that this algorithm not only achieves
a theoretical guarantee, but also it beats the heuristic algorithms
for maximizing mutual information. Finally, we examine the performance
of this algorithm on the vocabulary compression problem itself, and
confirm the effectiveness of the algorithm in producing a high-quality
solution for vocabulary compression large scale learning tasks.

{\bf \noindent Organization.} In the remainder of this section we
review related previous works and introduce the problem formally along
with appropriate notation. Then in Section~\ref{sec:algorithms}, we
introduce the novel compression algorithm and corresponding
theoretical guarantees as well as analysis of some basic heuristic
baselines. In Section~\ref{sec:empirical}, we present our empirical
evaluation of optimizing the mutual information objective as well as
an end-to-end learning task.


\subsection{Related Work}
{\bf Feature Clustering:}
The use of vocabulary compression has been studied previously,
especially in text classification applications where it is commonly
known as feature (or word) clustering. In particular,
\citet{baker1998} and \citet{slonim2001} both propose agglomerative
clustering algorithms, which start with singleton clusters that are
iteratively merged using a Jenson-Shannon divergence based function to
measure similarity between clusters, until the desired number of
clusters is found. Both algorithms are greedy in nature and do not
provide any guarantee with respect to a global objective. In
\citep{dhillon2003}, the authors introduce an algorithm that
empirically performs better than the aforementioned methods and that
also seeks to optimize the same global mutual information objective
that is analyzed in this work.  Their proposed iterative algorithm is
guaranteed to improve the objective at each iteration and arrive at a
local minimum, however, no guarantee with respect to the global
optimum is provided. Furthermore, each iteration of the algorithm
requires $O(mn)$ time (where $m$ is the size of the compressed
vocabulary) and the number of iterations is only guaranteed to be
finite (but potentially exponential). Later in this work, we compare
the empirical performance of this algorithm with our proposed method.

\begin{figure}[hbt]
	\centering
	\includegraphics[width=0.8\textwidth]{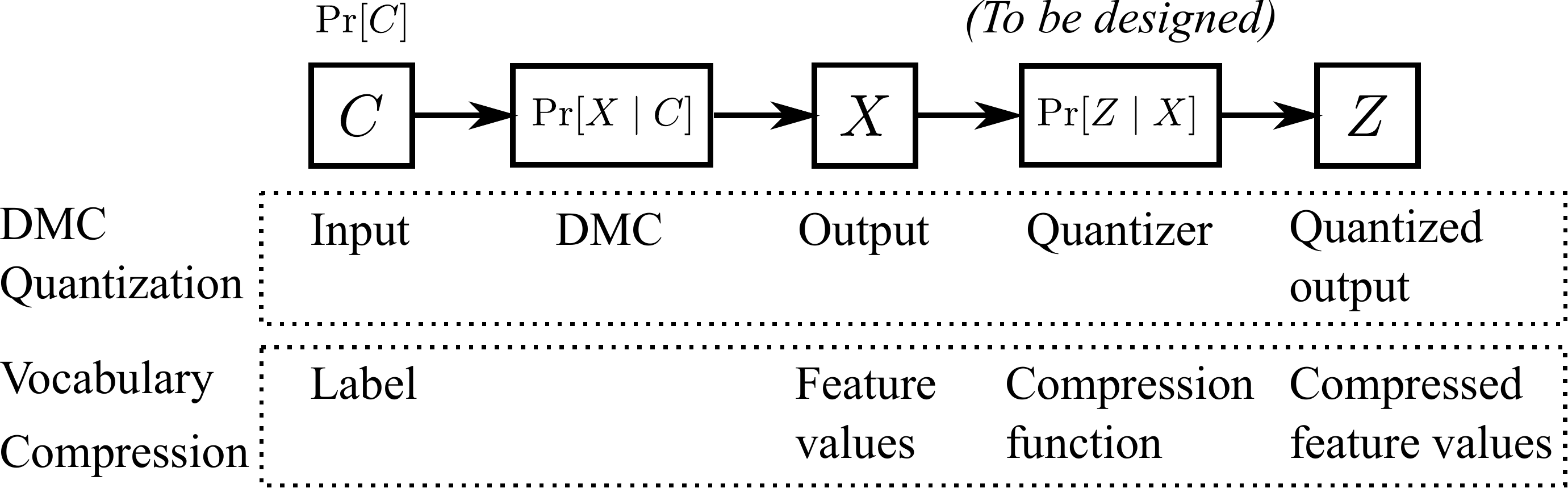}
	\caption{Translation of terminologies of the DMC quantizer design problem and the feature compression problem.}
	\label{fig:translation_dmc_vocab}
\end{figure}

{\bf Compression in Discrete Memoryless Channels:}
An area from information theory that is closely related to our
vocabulary compression problem, and which our algorithm draws
inspiration from, is compression in a discrete memoryless channels
(DMC) \citep{cicalese2018,zhang2016,iwata2014}. In this problem, we
assume there is a DMC which (in machine learning terminology) receive
a class label and produces a corresponding categorical feature value
drawn according to a fixed underlying distribution. The goal is to
design a \emph{quantizer} that maps the space of categorical features
in lower cardinatility set, while preserving as much of the mutual
information between the class label and newly constructed vocabulary.
In Figure~\ref{fig:translation_dmc_vocab}, we present a diagram that
illustrates the DMC quantization problem and vocabulary compression
problem as well as the translation of terminologies of these two
problems.  The results of \citet{kurkoski2014} are of particular
interest, as they show a cubic-time dynamic programming based
algorithm is able to provide an optimal solution in the case of binary
labels.  Following this work, \citet{iwata2014} improve the
computational complexity of this approach to quadratic time using the
SMAWK algorithm~\citep{aggarwal1987geometric}.  Such algorithms are
useful in the smaller scale regime, however, the use of a cubic-  or
even quadratic-time algorithm will be infeasible for our massive
vocabulary size use cases. Finally, \citet{mumey2003} shows that in
the general case of greater than two class labels, finding the optimal
compression is NP-complete. In this work, we will be focusing on the
binary label setting.

{\bf Feature Selection:}
A related method for dealing with very large vocabularies is to do
feature selection, in which we simply
select a subset of the vocabulary values and remove the rest (see
\citep{guyon2003} and the
many references therein). One can view this approach as a
special case of vocabulary compression, where we are restricted to
only singleton ``clusters''. Restricting the problem by selecting a
subset of the vocabulary may have some benefits, such as potentially
simplifing the optimization problem and the use of a simple filter to
transform data at inference time. However, the obvious downside to
this restriction is the loss of information and potentially poorer
learning performance (see introduction of \citep{jiang2011}). In this
work we focus on the more general vocabulary compression setting.

{\bf Other Feature Extraction Approaches:}
Clustering features in order to compress a vocabulary is only one
approach to lower dimensional feature extraction. There are of course
many classical approaches to feature extraction (see Chapter 15 of
\citep{mohri2018}), such as learning linear projections (e.g.,
Principle Component Analysis, Linear Discriminant Analysis) or
non-linear transformations (e.g., Locally Linear Embeddings, ISOMAP,
Laplacian Eigenmaps). However, these classical methods generally incur
more than quasilinear computational cost, for both learning and
the application the transformation, and are not feasible for our setting.

\subsection{Notation}

In the vocabulary compression problem we are given a correlated pair
of random variables $X$ (a categorical feature) and $C$ (a label), where $X\in
\{1,2,\dots,n\}$ and $C\in \{0,1\}$.  We aim to define a random
variable $Z\in \{1,2,\dots,m\}$ as a function of $X$ that maximizes
the mutual information with the label $C$, i.e.,  $I(Z;C)$, where for
general random variables $A$ and $B$ taking values in $\mathcal{A}$
and $\mathcal{B}$, respectively,
\begin{equation}
	I(A;B) = \sum_{A \in \mathcal{A}} \sum_{B \in \mathcal{B}} \prob{A,B} \log
	\Big(\frac{\prob{A, B}}{\prob{A}\prob{B}} \Big) \,.
\end{equation}
Note that $Z$ is a function of $X$ and hence we have $I(Z;C)\leq
I(X;C)$. If we let $m \geq n$, $Z=X$ maximizes the mutual
information $I(Z;C)$. We are interested in the
nontrivial case of $m \ll n$. Intuitively, we are compressing the
vocabulary of feature $X$ from size $n$ to a smaller size $m$, while
preserving the maximum amount of information about $C$.

\section{Algorithm and Analysis}
\label{sec:algorithms}
In this section, we first show how to transform the original objective
into a  set function and then prove that this set function is in
fact submodular. Next, we describe the components of a quasi-linear
and parallelizable algorithm to optimize the objective. Finally, we
consider a few simple intuitive baselines and show that they may
create features that fail to capture any mutual information with the
label.

\subsection{Objective Transformation}
Without loss of generality assume $\prob{C=0|X=i}$ for $i \in
\{1,\ldots,n\}$ is sorted in increasing order. Once the feature values
are sorted in this order, Lemma 3 of \citet{kurkoski2014} crucially
shows that in the optimum solution each value of $Z$ corresponds to a
consecutive subsequence of $\{1,\ldots,n\}$ --- this is a significant
insight that we take from the quantization for DMC literature. Thus,
we will cluster consecutive feature values into $m$ clusters, with
each cluster corresponding to a value in the compressed vocabulary of
$Z$.
Formally, define a function $F(S):
2^{\{1,\dots,n-1\}} \rightarrow \REAL$ as follows: Let
$S=\{s_1,\dots,s_{m-1}\}$, and assume $s_1<s_2<\dots<s_{m-1}$. For
simplicity, and without any loss in quality, we set $s_0=0$ and
$s_{m}=n$.  Let $Z$ be a random variable constructed from $X$ that has
value $i$, if and only if $s_{i-1} < X \leq s_{i}$. We define $F(S)=
I(Z;C)$.  Notice that we have 
\begin{equation*}
  \max_{S\subseteq\{2\dots n-1\} \colon |S| = m-1}F(S) = \max_{Z}I(Z;C) \,,
\end{equation*}
where $Z$ is a function of $X$ with vocabulary size $m$. The non-negativity of mutual information implies that the function $F(S)=I(Z;C)$ is always non-negative~\cite[p.~28]{Cover2006}. The monotonicity is equivalent to $I(Z_1; C)\le I(Z_2; C)$ for any $S_1\subseteq S_2\subseteq \{1,\dots,n-1\}$, where $Z_1$ and $Z_2$ are the random variables constructed from $S_1$ and $S_2$, respectively. Since $S_2$ represents a subdivision of $S_1$, $Z_1$ is a function of $Z_2$. By the data-processing inequality, we have $I(Z_1;C)\le I(Z_2;C)$~\cite[p.~34]{Cover2006}. In the
following section, we show that the function $F(S)$ is in fact
submodular.

\subsection{Submodularity of \texorpdfstring{$F(S)$}{F(S)}} \label{sec:Sub}

For a set $S\subseteq \{1,\dots,n-1\}$ and a number $s\in
\{1,\dots,n-1\}\setminus S$ we define $\Delta_s F(S) = F(S\cup \{s\})
- F(S)$. Let $s'$ be the item right before $s$ when we sort $S\cup
\{s\}$. Note that, the items that are mapped to $s'$ by $F(S)$ are
either mapped to $s'$ or $s$ by $F(S\cup \{s\})$. We first observe the
following useful technical lemma (the proof of all lemmas can be found
in the supplement).
\begin{lemma}
	\label{lem:delta_F}
Define the quantities $p=\prob{Z=s'}$, $q=\prob{Z=s}$,
$\x=\prob{C=0|Z=s'}$ and $\y=\prob{C=0|Z=s}$,
	then the following equality holds
	\begin{equation}
\Delta_s F(S) = {pf(\x)+ q f(\y) } - (p+q)f\big(\frac{p\x+q\y}{p+q}\big) \,,
\label{eq:delta1}
         \end{equation}
	 where $f(\cdot)$ the following convex function over $(0,1)$:
\begin{equation}
f(t) =  t\log\frac{t}{\prob{C=0}} +(1-t)\log\frac{1-t}{\prob{C=1}} \,.
\label{eq:f}
\end{equation}
\end{lemma}
Next, we provide several inequalities that precisely analyze expressions of the
same form as $\eqref{eq:delta1}$ with various values of $\x, \y, p$ and
$q$.
\begin{lemma}\label{lm:firstProperty}
	Pick $\x \leq \y \leq \z \in \REAL$, and $p\in [0,1]$. Let $q=1-p$ and let $f$ be an arbitrary convex function. We have
	\begin{equation*}
	 pf(\x) +q f(\y) - f(p\x+q\y) \leq  
		 pf(\x) +q f(\z) - f(p\x+q\z). 
	\end{equation*}
\end{lemma}
Replacing $p$ and $q$ in Lemma \ref{lm:firstProperty} with
$\frac{p}{p+q}$ and $\frac{q}{p+q}$ and multiplying both sides by
$p+q$ implies the following corollary.
\begin{corollary}\label{cr:firstProperty1}
Pick $\x\leq \y\leq \z \in \REAL$, and $p,q\in \REAL^+$. Let $f$ be an arbitrary convex function. We have
	\begin{equation*}
	{pf(\x) +q f(\y) } - (p+q)f\big(\frac{p\x+q\y}{p+q}\big) \leq  
	{pf(\x) +q f(\z)}  - (p+q)f\big(\frac{p\x+q\z}{p+q}\big). 
	\end{equation*}
\end{corollary}
Similarly, we have the following corollary (simply by looking at $f(-x)$ instead of $f(x)$). 
\begin{corollary}\label{cr:firstProperty2}
	Pick $\z \leq \x\leq \y \in \REAL$, and $p,q\in \REAL^+$. Let $f$ be an arbitrary convex function. We have
	\begin{equation*}
	{pf(\x) +q f(\y) } - (p+q)f\big(\frac{p\x+q\y}{p+q}\big)  \leq
	{pf(\z) +q f(\y)}  - (p+q)f\big(\frac{p\z+q\y}{p+q}\big). 
	\end{equation*}
\end{corollary}
We require one final lemma before proceeding to the main theorem.
\begin{lemma}\label{lm:secondProperty}
	Pick $\x, \y \in \REAL$, and $p,q,q'\in (0,1]$ such that $q<q'$. Let $f$ be an arbitrary convex function. We have
	\begin{equation*}
	pf(\x) +q f(\y)  - (p+q)f\big(\frac{p\x+q\y}{p+q}\big) \leq  
	pf(\x) +q' f(\y)  - (p+q')f\big(\frac{p\x+q'\y}{p+q'}\big). 
	\end{equation*}
\end{lemma}

\begin{theorem}[Submodularity]\label{thm:submodular}
	For any pair of sets $S_1\subseteq S_2 \subseteq
	\{1,\dots,n-1\}$ and any $s\in \{1,\dots,n-1\}\setminus S_2$ we have
	$$\Delta_s F(S_1) \geq \Delta_s F(S_2).$$
\end{theorem}
\begin{proof}
\begin{figure}
	\centering
	\includegraphics[width=0.6\textwidth]{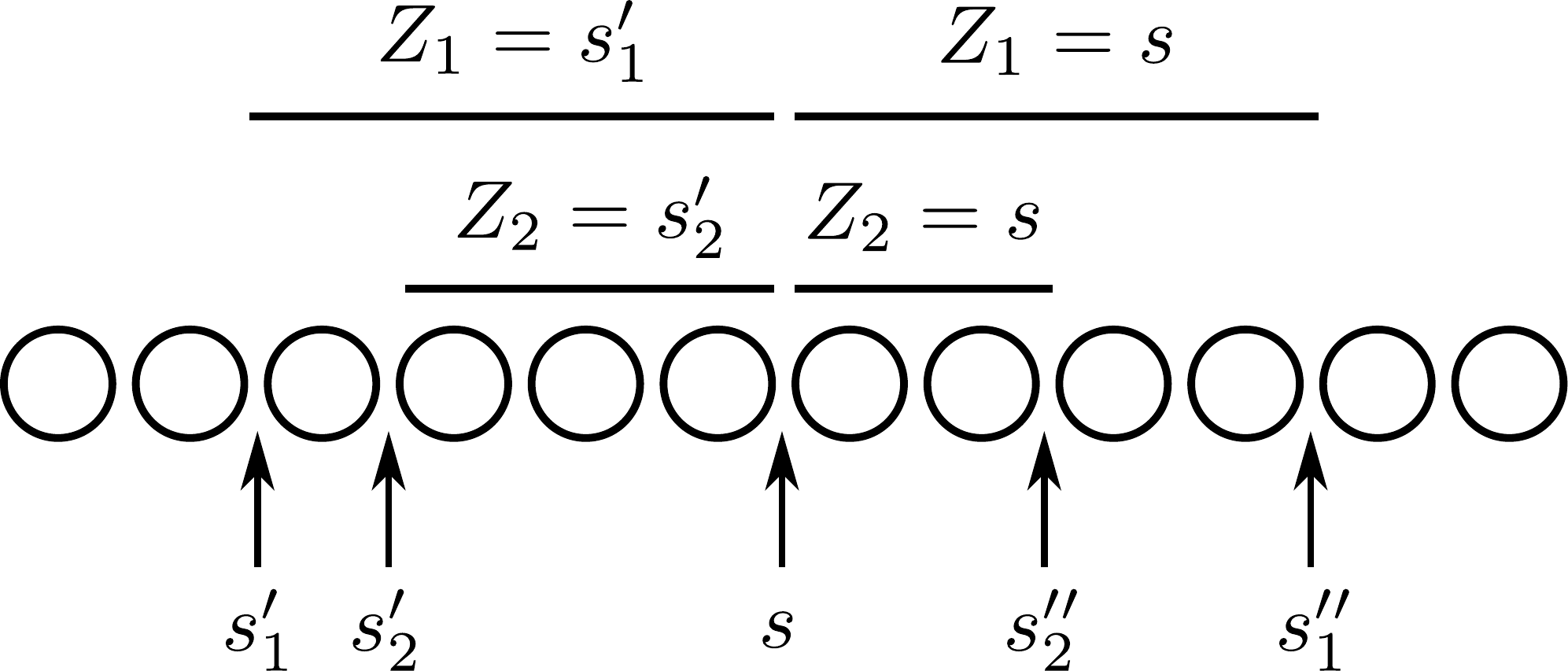}
	\caption{Illustration of boundaries used in proof of
	Theorem~\ref{thm:submodular}.}
\label{fig:Sorder}
\end{figure}
Let $s'_1$ and $s''_1$ be the items right before and right after $s$
when we sort $S_1\cup \{s\}$. Also, let $Z_1$ and $Z'_1$ be
the random variables corresponding to $F(S_1\cup \{s\})$ and
$F(S_1)$ respectively.
Similarly let $s'_2$ and $s''_2$ be items right before and right after
$s$ when we sort $S_2\cup \{s\}$, and let $Z_2$ and $Z'_2$ be
the random variables corresponding to $F(S_2\cup \{s\})$ and
$F(S_2)$ respectively.

Let us set $p_1=\prob{Z_1=s'_1}$, $q_1=\prob{Z_1=s}$,
$\x_1=\prob{C=0|Z_1=s'_1}$ and $\y_1=\prob{C=0|Z_1=s}$.
Similarly let us set $p_2=\prob{Z_2=s'_2}$,
$q_2=\prob{Z_2=s}$, $\x_2=\prob{C=0|Z_2=s'_2}$ and
$\y_2=\prob{C=0|Z_2=s}$.  Note that since $S_1\subseteq S_2$,
we have $s'_1,s''_1 \in S_2$ and hence we have $s'_2\geq s'_1$
and $s''_2\leq s''_1$ (see Figure~\ref{fig:Sorder}).
Therefore, we have following set of inequalities
\begin{align} 
p_2=\prob{Z_2=s'_2} &\leq \prob{Z_1=s'_1} = p_1 \,, \label{eq:p1,p2}\\
q_2=\prob{Z_2=s} &\leq \prob{Z_1=s} =q_1 \,. \label{eq:q1,q2}
\end{align}
Since in the definition of $F(\cdot)$ the elements are ordered
by $\prob{C=0|X=x}$, we have the following set of inequalities
\begin{align}
\x_1=\prob{C=0|Z_1=s'_1} &\leq \prob{C=0|Z_1=s} = \y_1 \,, \label{eq:x1,y1}
\\ 
\x_2=\prob{C=0|Z_2=s'_2} &\leq \prob{C=0|Z_2=s} = \y_2 \,, \label{eq:x2,y2}
\\ 
\x_1=\prob{C=0|Z_1=s'_1} &\leq \prob{C=0|Z_2=s'_2} = \x_2 \,, \label{eq:x1,x2}
\\ 
\y_2=\prob{C=0|Z_2=s} &\leq \prob{C=0|Z_1=s} = \y_1 \,. \label{eq:y1,y2}
\end{align}
Finally, we have 
\begin{align*}
& \Delta_s F(S_2) \\
 \overset{(a)}{=}{}& 	p_2f(\x_2) +q_2 f(\y_2)  -
	(p_2+q_2)f\big(\frac{p_2\x_2+q_2\y_2}{p_2+q_2}\big) \\
 \overset{(b)}{\leq}{}& p_2f(\x_2) +q_2 f(\y_1) -
	(p_2+q_2)f\big(\frac{p_2\x_2+q_2\y_1}{p_2+q_2}\big)  \\
 \overset{(c)}{\leq}{}& p_2f(\x_1) +q_2 f(\y_1) -
	(p_2+q_2)f\big(\frac{p_2\x_1+q_2\y_1}{p_2+q_2}\big) \\
 \overset{(d)}{\leq}{}& p_1f(\x_1) +q_2 f(\y_1) -
	(p_1+q_2)f\big(\frac{p_1\x_1+q_2\y_1}{p_1+q_2}\big)\\
 \overset{(e)}{\leq}{}& p_1f(\x_1) +q_1 f(\y_1) -
	(p_1+q_1)f\big(\frac{p_1\x_1+q_1\y_1}{p_1+q_1}\big) \\
 \overset{(f)}{=}{}& \Delta_s F(S_1) \,,
\end{align*}
where $(a)$ and $(f)$ follow from equality~\ref{eq:delta1}, $(b)$
follows from Corollary~\ref{cr:firstProperty1} and inequalities
\eqref{eq:y1,y2} and \eqref{eq:x2,y2}, $(c)$ follows from
Corollary~\ref{cr:firstProperty2} and inequalities \eqref{eq:x1,x2} and
\eqref{eq:x1,y1}, $(d)$ follows from Lemma~\ref{lm:secondProperty} and
inequality \eqref{eq:p1,p2}, and $(e)$ follows from
Lemma~\ref{lm:secondProperty} and inequality \eqref{eq:q1,q2}.
This completes the proof.
\end{proof}

\subsection{Submodular Optimization Algorithms}\label{sec:Alg}
Given that we have shown $F(\cdot)$ is submodular, we now show two
approaches to optimization: a single machine algorithm
that runs in time $O(n \log n)$ as well as an algorithm  which allows
the input to be processed in a distributed fashion, at the cost of an
additional logarithmic factor in running time.

{\bf Single Machine Algorithm:}
We will make use of
a $1 - 1/e - \epsilon$ approximation algorithm for submodular
maximization that makes only $O(n)$ queries to $\Delta_s F(S)$
\citep{mirzasoleiman2015}. First, fix an arbitrary small constant
$\epsilon$ (this
appears as a loss in the approximation factor as well as in the
running time). The algorithm starts with an empty solution set and
then proceeds in $m$ iterations where,
 in each iteration, we sample a set of
$\frac {n \log { 1/  \epsilon}}{m}$ elements uniformly at random from
the elements that have not been added to the solution so far and then add the
sampled element with maximum marginal increase to the solution.

In general, we may expect that computing $\Delta_s F(S)$ requires at
least $\Omega(|S|)$ time, which might be as large as $m$. However, we note
that the algorithm of \citep{mirzasoleiman2015} (similar to most
other algorithms for submodular maximization) only queries $\Delta_s
F(S)$ for incrementally growing subsets of the final
solution $S$. In that case, we can compute each incremental value of
$\Delta_s F(S)$ in logarithmic time using a data structure that costs
$O(n\log n )$ time to construct (see Algorithm~~\ref{alg:query}). By using
this query oracle, we do not require communicating the whole set $S$
for every query. Moreover, we use a red-black tree to maintain $S$,
and hence we can search for neighbors ($s'$ and $s''$) in logarithmic time.
Thus, combining the submodular maximization algorithm that requires
only a linear number of queries with the logarithmic time query oracle
implies the following theorem.
\begin{theorem}
For any arbitrary small $\epsilon > 0$, there exists a $(1 - 1/e - \epsilon)$-approximation algorithm for vocabulary
compression that runs in $O(n\log n)$ time.
\end{theorem}

\begin{algorithm}[t]
	\noindent
	\textbf{Procedure: Initialization}\\
	\noindent
	\textbf{Input:} Sorted list of probabilities $\prob{C=0|X=x_i}$ and probabilities $\prob{X=x_i}$.
	\begin{algorithmic}[1]
		\STATE Initiate a red-black tree data structure $S$.
		\STATE Insert $0$ and $n$ into $S$.
		\STATE $p_{<0}\leftarrow 0$
		\STATE $p_{C=0|<0}\leftarrow 0$
		\FOR {$i=1$ to $n$}
			\STATE $p_{<i}\leftarrow p_{<i-1}+\prob{x=x_i}$
			\STATE $p_{C=0|<i}\leftarrow \frac{p_{C=0|<i-1} \times p_{<i-1} + \prob{C=0|X=x_i} \times \prob{x=x_i}}{p_{<i}}$.
		\ENDFOR
	\end{algorithmic}
\noindent
	 \textbf{Procedure: Query $\Delta_s F(S)$}\\
\noindent
	\textbf{Input:} A number $s \in	\{1,\dots,n-1\}\setminus S$	
	\begin{algorithmic}[1]
		\STATE $s'\leftarrow$ largest element smaller than $s$ in $S$.
		\STATE $s''\leftarrow$ smallest element larger than $s$ in $S$.
		\STATE $p=p_{<s} - p_{<s'}$
		\STATE $q=p_{<s''} - p_{<s}$
		\STATE $\x=\frac{p_{C=0|<s}\times p_{<s} - p_{C=0|<s'}\times p_{<s'}}{p}$ 
		\STATE $\y=\frac{p_{C=0|<s''}\times p_{<s''} - p_{C=0|<s}\times p_{<s}}{q}$ 
		\STATE Return ${pf(\x) +q f(\y) } -
		(p+q)f\big(\frac{p\x+q\y}{p+q}\big)$ where $f(\cdot)$
		is defined by equation~\eqref{eq:f}.
	\end{algorithmic}

	\noindent
	 \textbf{Procedure: Insert $s$ to $S$}\\
	\noindent
	\textbf{Input:} A number $s  \in  \{1,\dots,n-1\}\setminus S$	
	\begin{algorithmic}[1]
		\STATE Insert $s$ into $S$. \\
	\end{algorithmic}
	
	\caption{Data structure  to compute $\Delta_s F(S)$}
	\label{alg:query}
\end{algorithm}

{\bf Distributed Algorithm:} Again, fix an arbitrary small number
$\epsilon>0$ (for simplicity assume $\epsilon k$ and
$\frac{n}{\epsilon k}$ are integers). In this distributed
implementation we use $\epsilon k$ machines, requires
$O(\frac{n}{\epsilon k})$ space per machine, and uses a
logarithmic number of rounds of computation.

To define our distributed algorithm we start with the
(non-distributed) submodular
maximization algorithm of~\citep{badanidiyuru2014fast}, which
provides a $1-1/e-\epsilon$ approximate solution using $O(n \log n)$
queries to the submodular function oracle.
The algorithm works by defining a decreasing sequence of thresholds
$w_0,w_1,\dots,w_{\log_{\frac{1}{1-\epsilon}}(n)}$, where $w_0$ is the
maximum marginal increase of a single element, and $w_i=(1-\epsilon)^i w_0$.
The algorithm proceeds in $\log_{\frac{1}{1-\epsilon}}(n)$ rounds,
where in round $i$ the algorithm iterates over all elements and
inserts an element $s$ into the solution set $S$ if $\Delta_s F(S)\geq
w_i$. The algorithm stops once it has selected $k$ elements or if it
finishes $\log_{\frac{1}{1-\epsilon}}(n)$ rounds, whichever comes
first.
As usual, this algorithm
only queries $\Delta_s F(S)$ for incrementally growing subsets of the
final solution $S$, and hence we can use Algorithm~\ref{alg:query} to
solve vocabulary compression in $O(n \log^2 n)$ time.

Now, we show how to distribute this computation across multiple
machines. First, for all $j\in \{1,\dots,\epsilon k-1\}$, we select
the $(j\frac{n}{\epsilon k})$-th element and add it to the solution set
$S$. This decomposes the elements into $\epsilon k$ continuous subsets
of elements, each of size $\frac{n}{\epsilon k}$, and each of which we
assign to one machine. Note that  $\Delta_s F(S)$ only depends on the
previous item and next item of $s$ in $S$ and, due to the way that we
created the initial solution set $S$ and decomposed the input
elements, the previous item and next item of $s$ are always both
assigned to the same machine as $s$. Hence each machine can compute
$\Delta_s F(S)$ locally. However, we assigned the first $\epsilon k
-1$ to the solution set blindly and their  marginal gain may be very
small. Intuitively, we are potentially throwing away some of our
selection budget for the ease of computation. Next we show that by
forcing these elements into the solution we do not lose more than
$\epsilon$ on the approximation factor.

First of all, notice that if we force a subset of the element to be
included in the solution, the objective function is still submodular
over the remaining elements. That is, the marginal impact of an element
$s$ (i.e., $\Delta_s F(S)$) shrinks as $S$ grows. Next we show that
if we force $\epsilon k -1$ elements into the solution, it does not
decrease the value of the optimum solution by more than a
$(1-\epsilon)$ factor. This means that if we provide a
$(1-1/e-\epsilon)$-approximation to the new optimum solution, it is a
$(1-\epsilon)\times(1-1/e-\epsilon) \leq (1-1/e-2\epsilon)$
approximate solution to the original optimum.

Let $S^*$ be a solution of size $k$ that maximizes $F(\cdot)$.
Decompose $S^*$ into $\frac 1 {\epsilon}$ subsets of size $\epsilon
k$. Note that by submodularity the value of $F(S^*)$ is more than
the
sum of the marginal impacts of each $\frac{1}{\epsilon}$ subset
(given the remainder of the subsets).
Therefore, by the pigeonhole principle, the marginal impact of one of
these subsets of $S^*$ is at most $\epsilon F(S^*)$. If we remove this
subset from $S^*$ and add the $\epsilon k -1$ forced elements, we find
a solution of size (at most) $k$ that contains all of the forced
elements and has value at least $(1-\epsilon)F(S^*)$ as desired.
Hence, by forcing these initial $\epsilon k -1$ elements to be in the
solution we lose only an $\epsilon$ fraction on the approximation factor.

Now, to implement the algorithm of~\citep{badanidiyuru2014fast}, in iteration $i$, each machine independently finds and inserts all of its elements with marginal increase more than $w_i$. If the number of selected elements exceeds $k$, we remove the last few elements to have exactly $k$ elements in the solution. This implies the following theorem.

\begin{theorem}
	For any arbitrary small $\epsilon > 0$, there exists a $(1 - 1/e - \epsilon)$-approximation $(\log n)$-round distributed algorithm for vocabulary
	compression with $O(\frac n k)$ space per machine and $O(n\log^2 n)$ total work.
\end{theorem}

\subsection{Heuristic Algorithms}
In this subsection we review a couple of heuristics that can serve as
simple alternatives to the algorithm we suggest and show that they
can, in fact, fail entirely for some inputs. We also provide an
empirical comparison to these, as well as the algorithm of
\citet{dhillon2003}, in Section~\ref{sec:empirical}.

\paragraph{Bucketing Algorithm:} This algorithm splits the range of
probabilities $[0,1]$ into $k$ equal size intervals $[0,1/k),
[1/k,2/k), \dots, [(k-1)/k,1]$. Then it uses these intervals (or
buckets) to form the compressed vocabulary. Specifically, each
interval represents all elements $i$ such that $\prob{C=0|X=i} \in
[(j-1)/k,j/k)$.  Note that there exists a set $S^b$ that such that
$F(S^b)$ correspond to the mutual information of the outcome of the
bucketing algorithm and the labels. First we show that it is possible to
give an upper bound on the mutual information loss, i.e., $I(X;C)-F(S^b)$.

\begin{theorem}
	Let $Z$ be the random variable provided by the bucketing algorithm. The total mutual information loss of the bucketing algorithm is bounded as follows.
	\begin{align*}
	I(X;C) - I(Z;C) \leq \Delta_{\max},
	\end{align*}
	where $\Delta_{\max} = \max_{j}\big(\max_{r\in [(j-1)/k,j/k)}
	f(r) - \min_{r\in [\frac{j-1}{k},\frac{j}{k})} f(r)\big)$ and $f(\cdot)$ is
	defined in equation~\eqref{eq:f}.
\end{theorem}
\begin{proof}
Note that as we showed in Subsection~\ref{sec:Sub} we have
\begin{align}\label{eq:FS}
F(S)&=\sum_{z\sim Z} \prob{Z=z} f(\prob{C=0|Z=z})\nonumber\\
&=\sum_{z\sim Z} \prob{Z=z} f\Big(E_{x\in z}\big[\prob{C=0|X=x}\big]\Big).
\end{align}
On the other hand we have
\begin{align}
&  I(X;C)
= \sum_{x\sim X} \prob{X=x} f(\prob{C=0|X=x}) \nonumber\\
= &\sum_{z\sim Z} \sum_{x\in z} \prob{X=x} f(\prob{C=0|X=x}) \nonumber\\
= &\sum_{z\sim Z} \prob{Z=z} \sum_{x\in z} \frac{\prob{X=x}}{\prob{Z=z}} f(\prob{C=0|X=x}) \nonumber\\
= &\sum_{z\sim Z} \prob{Z=z} E_{x\in z} \Big[f\big(\prob{C=0|X=x}\big)\Big].
\label{eq:IXC}
\end{align}

Let $j$ be the index of the interval corresponding to $z$. Then, by
convexity of $f(\cdot)$, we have
\[E_{x\in z} \big[f\big(\prob{C=0|X=x}\big)\big] \leq \max_{r\in
	[\frac{j-1}{k},\frac{j}{k})} f(r)\] and
\[f\big(E_{x\in z} \big[\prob{C=0|X=x}\big]\big) \geq \min_{r\in
	[\frac{j-1}{k},\frac{j}{k})} f(r)\,.\]
Therefore we have
\begin{multline*}
 E_{x\in z} \Big[f\big(\prob{C=0|X=x}\big)\Big] -
f\Big(E_{x\in z} \big[\prob{C=0|X=x}\big]\Big) \\ 
\leq \max_{r\in [(j-1)/k,j/k)} f(r) - \min_{r\in [(j-1)/k,j/k)} f(r) \leq \Delta_{\max} \,.
\end{multline*}
This together with Equations~\eqref{eq:FS} and~\eqref{eq:IXC} show that
$I(X;C) - F(S) \leq \sum_{z\sim Z} \prob{Z=z} \Delta_{\max} =
\Delta_{\max}$ and completes the theorem.
\end{proof}
The above theorem states that the information loss of the bucketing
algorithm is no more than how much $f(\cdot)$ changes within one
interval of size $1/k$. Note that this is an absolute loss and is not
comparable to the approximation guarantee that we provide submodular
maximization. The main problem with the bucketing algorithm is that it
is to some extent oblivious to the input data and, thus, will fail badly for
certain inputs as shown in the following proposition.
\begin{proposition}
There is an input $X$ to the bucketing algorithm such that
$I(X;C) > 0$ and $I(Z;C)=0$, where $Z$ is the output of the bucketing algorithm.
\end{proposition} 
\begin{proof}
Fix a number $j$. In this example for half of the items we have
	$\prob{C=0|X=x} = \frac {j+1/3}{k}$ and for the other half we
	have $\prob{C=0|X=x} = \frac {j+2/3}{k}$. We also set the
	probability of all values of $X$ to be the same, and hence
	$\prob{C=0}=\frac{j+0.5}{k}$. The mutual information of $X$
	with the label is non-zero since $\prob{C=0}\neq
	\prob{C=0|X=x}$. However, the bucketing algorithm merges all
	of the elements in the range $[\frac{j}{k},\frac{j+1}{k})$,
	thereby merging all values together giving us $I(Z;C)=0$ and
	completing the proof.
\end{proof}
Note, we can further strengthen the previous example by giving a tiny
mass to all buckets, so that all values do not collapse into a single
bucket.  However, still in this case, the bucketing method can only
hope to capture a tiny fraction of mutual information since the vast
majority of mass falls into a single bucket.

\paragraph{Frequency Based Filtering:} 
This is very simple compression method (more precisely, a
feature selection method) that is popular in practice. Given a
vocabulary budget, we compute a frequency threshold $\tau$ which we
use to remove any vocabulary value that appears in fewer than $\tau$
instances of our dataset and which results in a vocabulary of the
desired size.
Even though the frequency based algorithm is not entirely oblivious to
the input, it is oblivious to the label and hence oblivious to
conditional distributions. Similar to the bucketing algorithm with
a simple example in the following theorem we show that the frequency
based algorithm fails to provide any approximation gaurantee.
\begin{proposition}
	There is an input $X$ to the frequency based algorithm such
	that $I(X;C)> 0$ and  $I(Z;C)=0$, where $Z$ is the outcome of the frequency based algorithm.
\end{proposition} 
\begin{proof}
Assume we have $n=3k$ values for $X$, namely $x_1,\dots,x_n$. For all
	$i \in \{1,\dots,k\}$ define $\prob{X=x_i} = 2/n$,
	and for all $i \in \{k+1,\dots,n\}$ we have $\prob{X=x_i} =
	0.5/n$. Note that the first $k$ values are the most
	frequent values, however, we are going to define them such that
	they are independent of the label.

	For $\forall i \in \{1,\dots,k\}$ let $\prob{X=x_i|C=0} =
	1/2$; for $\forall i \in \{k+1,\dots,2k\}$ let
	$\prob{X=x_i|C=0} = 0$; and for $\forall i \in \{2k+1,\dots,3k\}$
	let $\prob{X=x_i|C=0} = 1$. Note that we have
	$\prob{C=0}=\frac 1 2$. Therefore the mutual information of
	the $k$ most frequent values with the label is zero, which
	implies for a certain vocabulary budget, and thereby
	frequency threshold, $I(Z;C)=0$. Observe that even if we merge
	the last $2k$
	values and use it as a new value (as opposed to ignoring
	them), the label corresponding the the merged value is $0$
	with probability half, and hence has no mutual information
	with the label. However, we have
	$I(X;C)=\sum_{x\sim X} \prob{X=x} f(\prob{C=0|X=x})
	= \sum_{i=k+1}^{2k} \frac{0.5}{n} = \frac{0.5 \times 2k}{3k} =
	\frac{1}{3} > 0$, which completes the proof.
\end{proof}

\section{Empirical Evaluation}
\label{sec:empirical}

\def\Submod{{\sc Submodular}}
\def\DMK{{\sc Divisive}}
\def\KY{KY}
\def\Freq{{\sc Frequency}}
\def\Bucket{{\sc Bucketing}}

In this section we report our empirical evaluation of the
optimization the submodular function $F(S)$ described in the previous
section.  All the experiments are performed using the Criteo click
prediction dataset \citep{criteo}, which consists of 37 million instances for training and 
4.4 million held-out points.
\footnote{
Note, we use the labeled training file from this challenge and
chronologically partitioned it into train/hold-out sets.
}
In addition to 13 numerical features, this dataset contains 26
categorical features with a combined total vocabulary of more than 28
million values.  These features have varying vocabulary sizes, from a
handful up to millions of values.  Five features, in particular, have
more than a million distinct feature values each.

In order to execute a mutual information based algorithm, we require
estimates of the conditional probabilities $\prob{C=0|X=x_i}$ and
marginal probabilities $\prob{X=x_i}$. Here, we simply use
the maximum likelihood estimate based on the empirical count, i.e.
given a sample of feature value and label pairs $\big((\hat
x_1, \hat c_1), \ldots, (\hat x_k, \hat c_k) \big)$, we have
\begin{align*}
\reallywidehat{\Pr} [ X=x_i] & =
 \frac{1}{k}\sum_{j=1}^k  {\bf 1}\{ \hat x_j = x_i \} \,, \\
\reallywidehat{\Pr} \left[C = 0 | X=x_i\right] & =
\frac{\frac{1}{k}\sum_{j=1}^k  {\bf 1}\{\hat c_j = 0 \wedge \hat x_j = x_i\}
}{\reallywidehat{\Pr} [ X=x_i]} \,.
\end{align*}
We note that such estimates may sometimes be poor, especially when
certain feature values appear very rarely. Evaluating more
robust estimates is outside the scope of the current study, but an
interesting direction for future work.

\subsection{Mutual information evaluation}

We first evaluate the ability of our algorithm to maximize the mutual
information retained by the compressed vocabulary and compare it to
other baseline methods.

In particular, we compare our algorithm to the iterative divisive
clustering algorithm introduced by \citet{dhillon2003}, as well as the
frequency filtering and bucketing heuristics introduced in the
previous section.  The divisive clustering algorithm resembles a
version of the $k$-means algorithm where $k$ is set to be the
vocabulary size and distances between points and clusters are defined
in terms of the KL divergence between the conditional distribution of
the label variable given a feature value and the conditional
distribution of the label variable given a cluster center.  Notice
that due to the large size of the dataset, we cannot run the dynamic
programming algorithm introduced by~\citet{kurkoski2014} which would
find the theoretically optimal solution.  For ease of reference, we
call our algorithm \Submod, and the other algorithms \DMK, \Bucket\ and
\Freq.


Note that our algorithm, as well as previous algorithms, seek to
maximize the mutual information between a \emph{single} categorical
variable and the label, while in the Criteo dataset we have several
categorical variables that we wish to apply a global vocabulary size budget
to. In the case of the \Freq\ heuristic, this issue is addressed by
sorting the counts of feature values across all categorical variables
and applying a global threshold.
In the case of \Submod, we run the
almost linear-time algorithm for each categorical variable to obtain a
sorted list of feature values and their marginal contributions to the
global objective.  Then we sort these marginal values and pick the
top-score feature values to obtain the desired target vocabulary size.
Thus, both \Submod\ and \Freq\ are able to naturally employ \emph{global}
strategies in order to allocate the total vocabulary budget across
different categorical features.

For the \DMK\ and \Bucket\ algorithms, a natural global allocation
policy is not available, as one needs to define an allocation of the
vocabulary budget to each categorical feature a priori.  In this
study, we evaluate two natural allocation mechanisms.  The
\emph{uniform allocation} assigns a uniform budget across all
categorical features, whereas the \emph{MI allocation} assigns a
budget that is proportional to the mutual information of the
particular categorical feature.




The original vocabulary of over 28 million values is compressed by a
factor of up to 2000.  Using the methods mentioned above, we obtain
vocabularies of size 10K, 20K, 40K, 80K, 120K and 160K.  Then we
compute the loss in average mutual information, which is defined as
follows: let $X_i$ denote the mutual information of uncompressed
categorical feature $i$ with the label and $Z_i$ denote mutual
information of the corresponding compressed feature, then
the average mutual information loss is equal to $(\sum_i X_i -
Z_i)/(\sum_j X_j)$.

\begin{figure}
	\centering
	\includegraphics[width=0.5\textwidth]{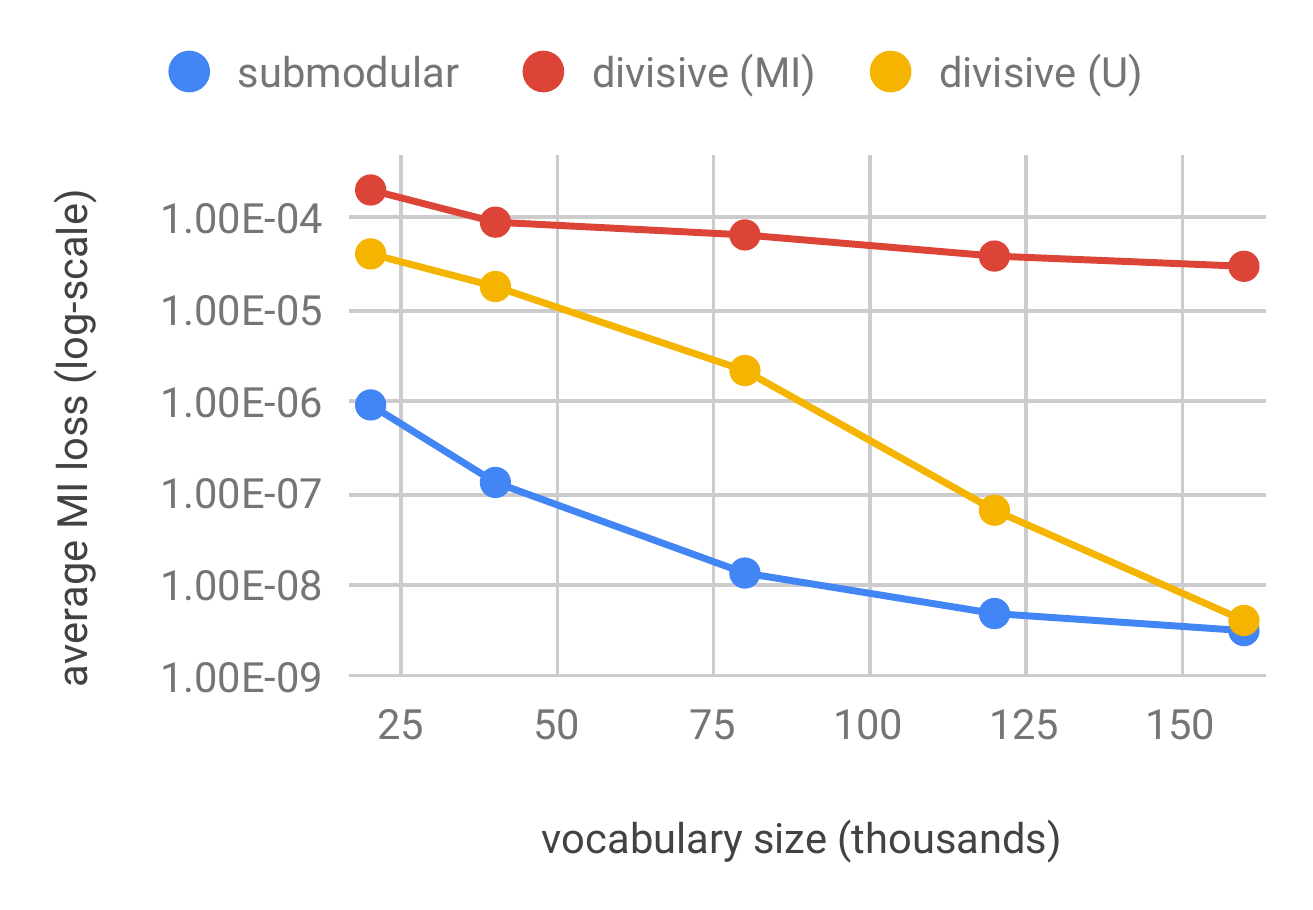}
\caption{The average mutual information loss of several compression
methods measured on to the Criteo dataset.}
\label{fig:mi_train}
\end{figure}


For the heuristic \Freq\ algorithm, the measured loss ranges from
0.520 (for budget of 160K) to 0.654 (for budget of 10K), while for
\Bucket\ the loss ranges from $5\times 10^{-6}$ to $5\times 10^{-3}$. As
expected, the mutual information based methods perform significantly
better, in particular, the loss for \Submod\ ranges from $9\times
10^{-7}$ to $3 \times 10^{-9}$ and consistently outperforms the \DMK\
algorithm (regardless of allocation strategy).
Figure~\ref{fig:mi_train} provides a closer look at the mutual
information based methods.
Thus, we find that not only is our method fast, scalable and exhibits
a theoretical $1-1/e$ lower bound on the performance, but that in
practice it maintains almost all the mutual information between data
points and the labels.




\begin{figure}
        \centering
        \includegraphics[width=0.5\textwidth]{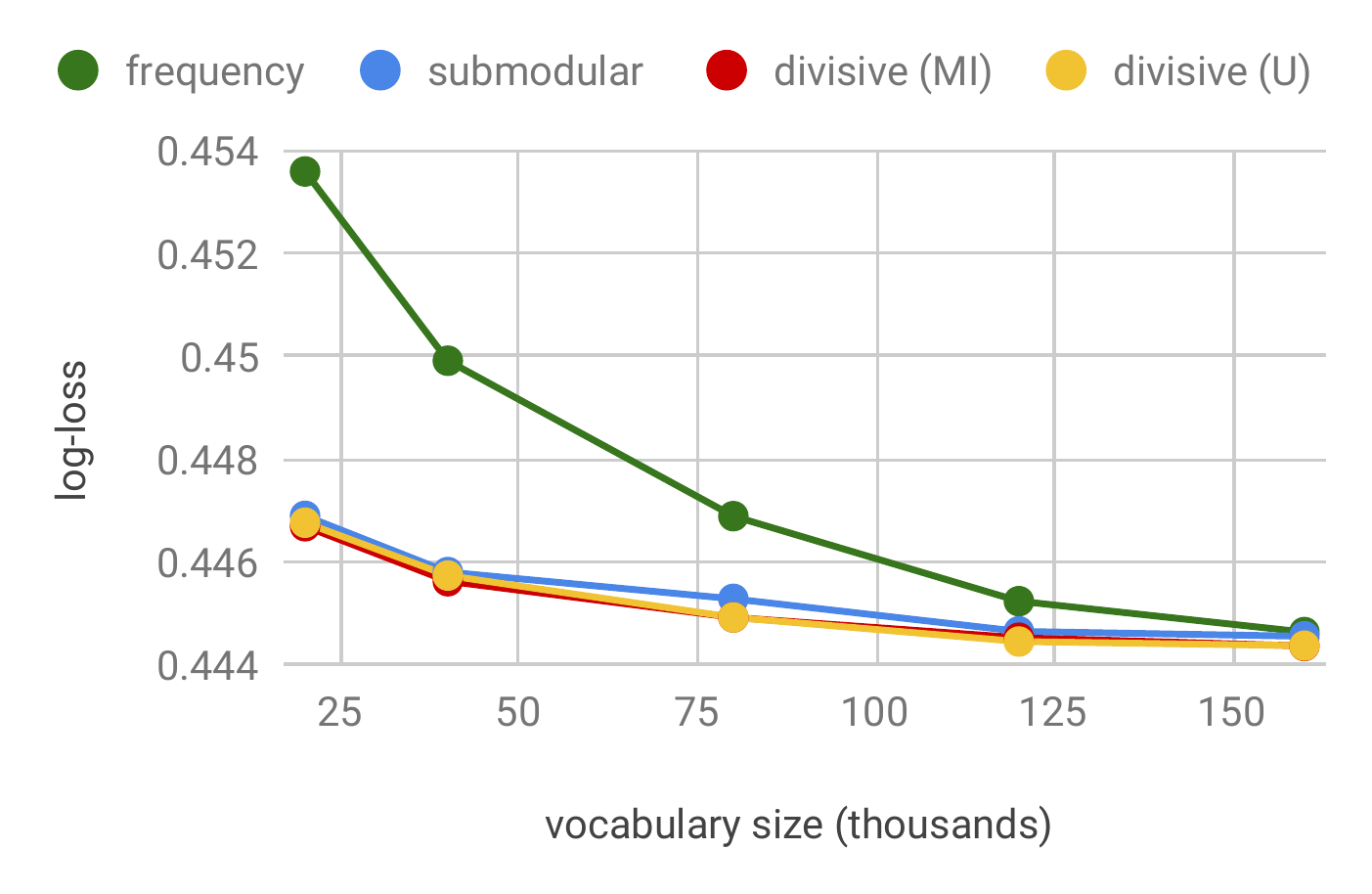}
\caption{The log-loss of a neural network model trained with
compressed vocabularies of several sizes and using several different
compression methods.}
\label{fig:logloss}
\end{figure}

\subsection{Learning evaluation}

Our focus thus far has been in optimizing the mutual information
objective. In this section we also evaluate the compressed
vocabularies in an end-to-end task to demonstrate its application in a
learning scenario.

Given a compressed vocabulary we train a neural network model on the
training split and measure the log-loss on the hold out set (futher
details in supplement Section~\ref{sec:emp_supplement}).\footnote{In
order to alleviate the potential issue of poor conditional/marginal
distribution estimates we initially start with only features values
that appear in at least 100 instances.} In Figure~\ref{fig:logloss} we
see that the mutual information based methods perform comparably to
each other and significantly outperform popular heuristic method
\Freq. We observe that our scalable quasi-linear compression algorithm
with provable approximation guarantees also performs competitively in
end-to-end learning.

\section{Conclusion}
\label{sec:conclusions}
In this work we have shown the first scalable quasi-linear compression
algorithm for maximizing mutual information with the label that also
exhibits and $1-1/e$ factor approximation guarantee. The algorithm, as
well as our insights into constructing a submodular objective function,
might be of interest in other applications as well (for example,
quantization in DMC). One future line of work is extending this work to
the multiclass (non-binary) setting.

\bibliographystyle{plainnat}
\bibliography{vocab_compression}

\clearpage
\onecolumn
\appendix
\section{Supplement}

\subsection{Proof of technical lemmas}
\begin{proof}[Proof of Lemma~\ref{lem:delta_F}]
Let $Z$ and $Z'$ be
the random variables corresponding to $F(S\cup \{s\})$ and $F(S)$
respectively. Note that we have
\begin{align*}
 F(S)
&= \sum_{z'\sim Z'} \sum_{c\in\{0,1\}}
	\prob{Z'=z',C=c}\log\frac{\prob{Z'=z',C=c}}{\prob{Z'=z'}\prob{C=c}}\\
	&=\sum_{z'\sim Z'} \prob{Z'=z'} 
	\sum_{c\in\{0,1\}}
	\prob{C=c|Z'=z'}\log\frac{\prob{C=c|Z'=z'}}{\prob{C=c}} \\
&=\sum_{z'\sim Z'} \prob{Z'=z'} f(\prob{C=0|Z'=z'}),
\end{align*}
where we have
\begin{equation*}
f(t) =  t\log\frac{t}{\prob{C=0}} +(1-t)\log\frac{1-t}{\prob{C=1}} \,,
\end{equation*}
which is a convex function over $t \in [0,1]$.
Next, we have
\begin{align*}\label{eq:delta0}
\Delta_s F(S)
={}& F(S\cup \{s\}) - F(S)\nonumber \\
={}& \sum_{z\sim Z} \prob{Z=z} f(\prob{C=0|Z=z}) - \sum_{z'\sim Z'} \prob{Z'=z'} f(\prob{C=0|Z'=z'}) \nonumber\\
={}& \prob{Z=s'} f(\prob{C=0|Z=s'}) + \prob{Z=s} f(\prob{C=0|Z=s}) \\
&- \prob{Z'=s'} f(\prob{C=0|Z'=s'}).
\end{align*}
Notice that $Z'=s'$ implies that $Z=s$ or $Z=s'$. Hence we have $\prob{Z'=s'} = \prob{Z=s'} + \prob{Z=s}$ and
\begin{align*}
\prob{C=0|Z'=s'} = \frac{\prob{Z=s'} \prob{C=0|Z=s'} +
	\prob{Z=s}\prob{C=0|Z=s}}{\prob{Z=s'}
	+ \prob{Z=s}} \,.
\end{align*}
Now, if we set $p=\prob{Z=s'}$, $q=\prob{Z=s}$,
$\x=\prob{C=0|Z=s'}$ and $\y=\prob{C=0|Z=s}$, and combine the previous
two inline equalities, we have
\begin{equation*}
\Delta_s F(S) =
{pf(\x)+ q f(\y) } - (p+q)f\big(\frac{p\x+q\y}{p+q}\big) \,.
\end{equation*}
\end{proof}

{\bf Some Basic Tools:}
In Lemmas \ref{lm:firstProperty} and \ref{lm:secondProperty} we show two basic properties of convex functions that later become handy in our proof.
We use the following property of convex functions to prove Lemma \ref{lm:firstProperty}. For any convex function $f$ and any three numbers $a<b<c$ we have 
\begin{align}\label{eq:convexDef1}
	\frac{f(b)-f(a)}{b-a} \leq \frac{f(c)-f(b)}{c-b}. 
\end{align}
Note that this also implies
\begin{align}\label{eq:convexDef2}
	\frac{f(c)-f(a)}{c-a} &= \frac{1}{c-a}\big(f(c)-f(b) + f(b)-f(a)\big)\nonumber\\
	&\leq \frac{1}{c-a}\Big(f(c)-f(b) + \frac{b-a}{c-b}\big(f(c)-f(b)\big)\Big) &\text{By Inequality \ref{eq:convexDef1}}\nonumber\\
	& = \frac{1}{c-a}\Big(\frac{c-b + b-a}{c-b}\big(f(c)-f(b)\big)\Big)\nonumber\\
	& = \frac{f(c)-f(b)}{c-b}.
\end{align}
Similarly we have 
\begin{align}\label{eq:convexDef3}
\frac{f(c)-f(a)}{c-a} &=\frac{1}{c-a}\big(f(c)-f(b) + f(b)-f(a)\big)\nonumber\\
&\geq \frac{1}{c-a}\Big(\frac{c-b}{b-a}\big(f(b)-f(a)\big) + f(b)-f(a)\Big)&\text{By Inequality \ref{eq:convexDef1}}\nonumber \\
&\geq \frac{1}{c-a}\Big(\frac{c-b+b-a}{b-a}\big(f(b)-f(a)\big) \Big)\nonumber\\
&= \frac{f(b)-f(a)}{b-a}. 
\end{align}

\begin{proof}[Proof of Lemma~\ref{lm:firstProperty}] 
	First, we prove 
	\begin{align}\label{eq:firstPropEq0}
		\frac{f(p\x+q\z)-f(p\x+q\y)}{q\z-q\y} \leq \frac{f(\z)-f(\y)}{\z-\y}.
	\end{align}
	Recall that $\x\leq \y\leq \z$, and $p+q=1$. Hence we have $p\x+q\y \leq p\x+q\z, \y \leq \z$. We prove Inequality \ref{eq:firstPropEq0} in two cases of $p\x+q\z \leq \y$, and $\y < p\x+q\z$.
	\\
	{\bf{Case 1.}} In this case we have $p\x+q\y \leq p\x+q\z \leq \y \leq \z$. we have
	\begin{align*}
		\frac{f(p\x+q\z)-f(p\x+q\y)}{q\z-q\y} &= \frac{f(p\x+q\z)-f(p\x+q\y)}{(p\x+q\z)-(p\x+q\y)}\\
		&\leq \frac{f(\y) - f(p\x+q\z)}{\y -(p\x+q\z)} &\text{By Inequality \ref{eq:convexDef1}}\\
		&\leq \frac{f(\z) - f(\y)}{\z-\y} &\text{By Inequality \ref{eq:convexDef1}}
	\end{align*}
	\\
	{\bf{Case 2.}} In this case we have $p\x+q\y \leq \y \leq p\x+q\z  \leq \z$. we have
	\begin{align*}
	\frac{f(p\x+q\z)-f(p\x+q\y)}{q\z-q\y} &= \frac{f(p\x+q\z)-f(p\x+q\y)}{(p\x+q\z)-(p\x+q\y)}\\
	&\leq \frac{f(p\x+q\z)-f(\y)}{(p\x+q\z)-\y} &\text{By Inequality \ref{eq:convexDef2}}\\
	&\leq \frac{f(\z)-f(\y)}{\z-\y} &\text{By Inequality \ref{eq:convexDef3}}.
	\end{align*}
	Next we use Inequality \ref{eq:firstPropEq0} to prove the lemma. By multiplying both sides of Inequality \ref{eq:firstPropEq0} by $q(\z-\y)$ we have 
	\begin{align*}
		{f(p\x+q\z)-f(p\x+q\y)} \leq {qf(\z)-qf(\y)}.
	\end{align*}
	By rearranging the terms and adding $pf(\x)$ to both sides we have
	\begin{align*}
		\big(pf(\x) +q f(\y) \big) - f(p\x+q\y) \leq  
		\big(pf(\x) +q f(\z) \big) - f(p\x+q\z),  
	\end{align*}
	as desired.
\end{proof}

\begin{proof}[Proof of Lemma~\ref{lm:secondProperty}]
	We have
	\begin{align*}
		\frac{p+q}{p+q'} f\big(\frac{p\x+q\y}{p+q}\big) + \frac{q'-q}{p+q'}f(\y) 
		&\geq f\big(\frac{p+q}{p+q'} \frac{p\x+q\y}{p+q} + \frac{q'-q}{p+q'}\y\big) &\text{By convexity}\\
		&=f\big(\frac{p\x+q\y}{p+q'} + \frac{q'-q}{p+q'}\y\big)\\
		&=f\big(\frac{p\x+q'\y}{p+q'}\big).
	\end{align*}
	By multiplying both sides by $p+q'$ we have
	\begin{align*}
	(p+q) f\big(\frac{p\x+q\y}{p+q}\big) + q'f(\y)-qf(\y) 
	&\geq (p+q'	)f\big(\frac{p\x+q'\y}{p+q'}\big).
\end{align*}
	By rearranging the terms and adding $pf(\x)$ to both sides we have
	\begin{align*}
	pf(\x) +q f(\y)  - (p+q)f\big(\frac{p\x+q\y}{p+q}\big) \leq  
	pf(\x) +q' f(\y)  - (p+q')f\big(\frac{p\x+q'\y}{p+q'}\big),
	\end{align*}
	as desired.
\end{proof}

\subsection{Empirical Evaluation Details}
\label{sec:emp_supplement}
We implement the neural network using TensorFlow and train it using the
AdamOptimizer \citep{abadi2016,kingma}. The following set of neural network
hyperparameters are tuned by evaluating 2000 different configurations on the
hold-out set as suggested by a Gaussian Process black-box optimization
routine.
\begin{center}
\begin{tabular}{l|c}
	{\bf hyperparameter} & {\bf search range} \\
	\hline
	hidden layer size & [100, 1280] \\
	num hidden layers & [1, 5] \\
	learning rate & [1e-6, 0.01] \\
	gradient clip norm & [1.0, 1000.0] \\
	$L_2$-regularization & [0, 1e-4]
\end{tabular}
\end{center}

\end{document}